\documentclass[sigconf]{aamas}

\usepackage{balance} % for balancing columns on the final page

\makeatletter
\gdef\@copyrightpermission{
  \begin{minipage}{0.2\columnwidth}
   \href{https://creativecommons.org/licenses/by/4.0/}{\includegraphics[width=0.90\textwidth]{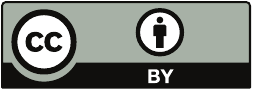}}
  \end{minipage}\hfill
  \begin{minipage}{0.8\columnwidth}
   \href{https://creativecommons.org/licenses/by/4.0/}{This work is licensed under a Creative Commons Attribution International 4.0 License.}
  \end{minipage}
  \vspace{5pt}
}
\makeatother

\setcopyright{ifaamas}
\acmConference[AAMAS 2027]{Proc.\@ of the 26th International Conference
on Autonomous Agents and Multiagent Systems (AAMAS 2027)}{3--7 May 2027}{Hanoi, Vietnam}{M.~Baldoni, F.~Fang, W.~Yeoh, N.~Yorke-Smith (eds.)}
\copyrightyear{2027}
\acmYear{2027}
\acmDOI{}
\acmISBN{}

\acmSubmissionID{223}

\submissionType{Research Paper Track}
\title{ToM-2 Understanding Statements for Teacher-Guided Preference Learning}

\author{Jack Mirenzi}
\email{jmirenzi@andrew.cmu.edu}
\orcid{0009-0004-2616-8714}
\affiliation{%
  \institution{Carnegie Mellon University}
  \city{Pittsburgh}
  \state{Pennsylvania}
  \country{USA}
}

\author{Henny Admoni}
\email{hadmoni@andrew.cmu.edu}
\orcid{0000-0003-1796-2196}
\affiliation{%
  \institution{Carnegie Mellon University}
  \city{Pittsburgh}
  \state{Pennsylvania}
  \country{USA}
}

\begin{abstract}
    Preference learning is a powerful tool for reward learning, but treats the teacher as a passive oracle that answers queries driven by the learner. This learner-guided query selection ignores the teacher's knowledge of the solution. We reverse the typical preference query selection method by having the \emph{teacher} create ordered preference pairs, since a teacher who knows the target can choose more informative pairs than a learner choosing its own queries. Exploiting this advantage requires the teacher to model what the learner knows, and that model goes stale when the learner's belief changes in ways the teacher cannot see, as when several teachers each teach without knowing what the others taught. We call this \emph{teacher-model drift}. In this work, we use second-order Theory-of-Mind (ToM-2) to have the learner create \emph{understanding statements} that correct the teacher's model of the learner's knowledge. In simulation, we show that a learner with a perfect teacher, one with perfect knowledge of the learner, needs about half as many rounds as learner-guided querying to learn the reward. Yet when teacher-model drift is induced by alternating between four co-teachers, the advantage erodes. We tested two types of understanding statement generation: the learner's mean belief and ToM-2 drift minimization. ToM-2 statements reduced the drift and recovered more of the lost advantage than mean statements. In a within-subjects study ($n = 101$) with a simulated ideal co-teacher, ToM-2 statements raised the learner's final alignment with the true reward above both uninformed and mean-informed teaching; on the regret participants removed in their own rounds, they outperformed uninformed teaching but not mean statements.
\end{abstract}

\newcommand{\BibTeX}{\rm B\kern-.05em{\sc i\kern-.025em b}\kern-.08em\TeX}

\newcommand{\Bt}{\mathcal{B}_t}
\newcommand{\BT}{\mathcal{B}^T_t}
\newcommand{\BhT}{\hat{\mathcal{B}}^T_t}
\newcommand{\RT}{R^T_t}
\newcommand{\xstar}{x^*}
\newcommand{\phit}{\phi_t}
\newcommand{\Deltaphi}{\Delta\phi_t}
\newcommand{\Va}{V(a)}
\newcommand{\Qa}{Q(a)}
\newcommand{\Rd}{\mathbb{R}^d}
\newcommand{\Sd}{S^{d-1}}

\begin{document}

%%% The following commands remove the headers in your paper. For final 
%%% papers, these will be inserted during the pagination process.

%\pagestyle{fancy}
%\fancyhead{}

%%% The next command prints the information defined in the preamble.

\maketitle 
\begin{figure}[t]
  \centering
  \hspace{-.5cm}\includegraphics[width=.95\linewidth]{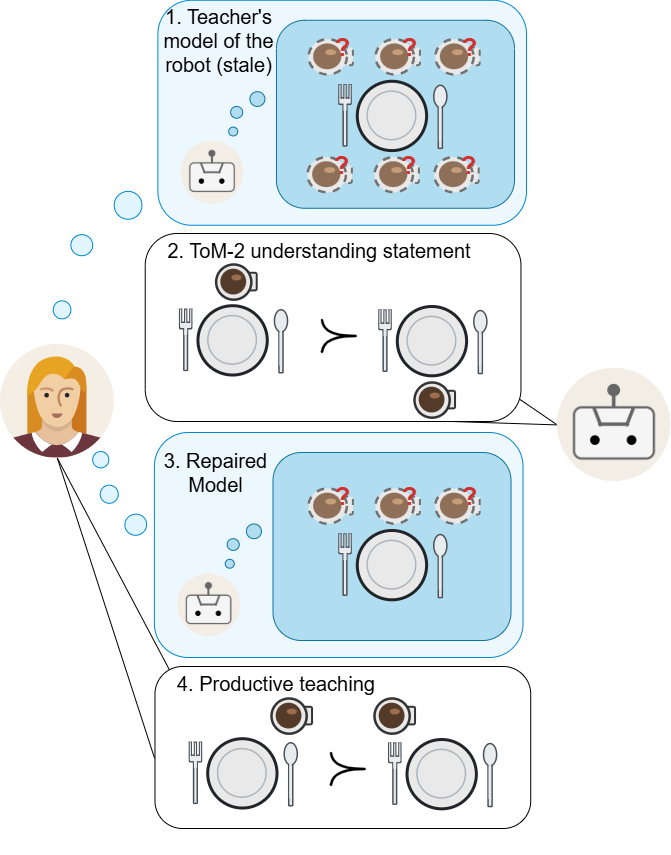}
  % \caption{The ToM-2 Learner recognizes that the teacher's model of the learner is stale: it fails to reflect knowledge the learner has already acquired. The learner issues an understanding statement, a preference pair that reveals this knowledge, repairing the teacher's model before the next query.}
  \caption{The ToM-2 learner recognizes that the teacher's model of it is stale: the teacher does not know the learner has already learned that the mug goes behind the plate. The learner issues an understanding statement that repairs this drift, and the teacher can then teach something new.}
  \Description{Four-panel diagram of a human teacher and a robot learner setting a table. Panel 1, in the teacher's thought bubble, shows her stale model of the robot, which is unsure where the cup goes. Panel 2 shows the robot's understanding statement: an arrangement with the cup behind the plate preferred over one with the cup in front. Panel 3 shows the teacher's repaired model, which now places the cup behind the plate and is unsure only whether it goes left or right. Panel 4 shows the teacher's next pair, which contrasts the cup behind and to the right with the cup behind and to the left.}
  \label{fig:diagram}
\end{figure}
\section{Introduction}
\label{sec:intro}

Robots increasingly learn reward functions from comparative preference feedback, casting the human as a passive oracle who labels learner-generated queries~\cite{sadigh-rss17,biyik-corl18}. This framing ignores the defining asymmetry of teaching: the teacher knows the objective and the learner does not. We would not ask a student to learn only from questions they themselves pose; the same pedagogy applies to robot reward learning. Letting the learner drive query selection forfeits the teacher's most valuable resource: a direct line to the goal.

The cost of this forfeiture is well understood. A learner choosing its own query does not know the answer in advance, so no query rule gains more than one bit per round in expectation, whereas a teacher who knows the target can choose a pair that rules out more than half of the learner's remaining hypotheses. That a target-aware teacher needs fewer examples than a learner querying on its own is a founding observation of machine teaching~\cite{goldman-jcss95,zhu-survey18}; we take it as the motivation for teacher-guided preference learning rather than as a result (Section~\ref{sec:analysis}).

For this teacher guidance to be effective, however, the teacher requires a model of what the learner currently knows. We call any divergence between this model and the learner's actual belief \emph{teacher-model drift}. Drift arises whenever the learner's belief changes in a way the teacher did not observe or predict, for instance when another teacher teaches in its absence or when the learner updates differently than the teacher assumes. Under drift, the teacher spends rounds on examples the learner has already absorbed, and its advantage erodes. Section~\ref{sec:liability} measures drift in bits, shows that it reduces the information each of the teacher's pairs delivers, and shows that understanding statements pay it back. Drift is especially likely with multiple teachers. Each teacher's model is built only from the rounds that teacher was present for, so it goes stale as soon as another teacher takes over, and it is wrong in whatever direction that teacher taught. Consider a household robot learning to set the table, taught across a week by several family members with no one present for the whole process: one teaches utensil placement on Monday, another corrects mug location on Thursday, then the first teacher returns and, unnecessarily, teaches mug placement. We study this alternating co-teaching as a controlled source of teacher-model drift.

To keep this model synchronized, we build on \emph{understanding statements}, a second-order Theory-of-Mind (ToM-2) mechanism in which the learner communicates information that repairs the teacher's model of what the learner knows~\cite{callaghan-arix25,callaghan-cogsci26}. We recast understanding statements for continuous preference-based reward learning: instead of a linguistic statement over enumerated features, the learner states an ordered pair of demonstrations. We use an ordered pair so that the statement mirrors the pair the teacher uses to teach the learner. 
We contribute:
\begin{enumerate}
    \item A formulation of understanding statements for continuous preference learning in which the statement is a preference constraint in the teacher's action space, with two selection policies: a mean statement that summarizes the learner's belief and a ToM-2 statement that targets the teacher's model error.
    \item A drift ledger: under noiseless hard cuts, teacher-model drift measured in bits changes by an exact amount with every event in co-teaching. A co-teacher's pair that the teacher did not see adds its bits, an understanding statement consistent with the learner's belief subtracts the bits it removes from the teacher's model, and drift bounds the bits a teacher's pair can lose.
    \item A simulation with four alternating teachers, for $d$ from 10 to 30, in which drift removes part of the teacher's advantage over learner-guided querying and ToM-2 statements recover more of it than mean statements.
    \item A within-subjects study ($n = 101$) in which human teachers alternated with a simulated co-teacher. ToM-2 statements raised the learner's final alignment over both uninformed and mean-informed teaching; on the regret removed by participants' own teaching, they improved over uninformed teaching but not over mean statements.
\end{enumerate}

\section{Related Works}
\label{sec:related}

\paragraph{Preference-based reward learning}
Active preference-based reward learning casts the human as an oracle answering learner-generated queries, selected to remove belief volume \cite{sadigh-rss17} or in batches \cite{biyik-corl18}. Volume removal was later shown to admit degenerate queries; information-gain acquisition dominates it and yields easier queries \cite{biyik-corl19}. These methods, including Reinforcement Learning from Human Feedback (RLHF) \cite{christiano-neurips17}, share the assumption that the human does not steer the curriculum. Active-learning label complexity is governed by how evenly queries can split the hypothesis space~\cite{dasgupta-nips05}. Because a learner scores queries against its own belief, learner-guided querying cannot suffer teacher-model drift; teacher guidance trades that immunity for more informative pairs, which is the trade-off this paper addresses.

\paragraph{Machine teaching}
Machine teaching studies how a target-aware agent selects examples to drive a learner toward a desired hypothesis. Its sample complexity is governed by the teaching dimension, which can be far below what an active learner needs~\cite{goldman-jcss95,zhu-survey18}. In reward learning, algorithmically selected demonstrations identify a target reward far faster than naturally chosen ones~\cite{cakmak-aaai12}, and robots have taught their policies to people by selecting demonstrations that are informative for the human's Inverse Reinforcement Learning (IRL), including through counterfactual reasoning over the human's current beliefs~\cite{lee-frontiers21,lee-iros22}. These results assume the teacher knows the learner's state. Work on teaching with an imperfect learner model relaxes this: iterative teachers with partial knowledge of the learner choose examples from its observed progress~\cite{liu-icml17} or actively query a black-box learner to estimate its state~\cite{liu-icml18}, interaction restores finite-time teaching when the learner's learning process is unknown~\cite{melo-ijcai18}, and an IRL teacher adapts each demonstration to the learner's current policy~\cite{kamalaruban-ijcai19}. In each case the teacher repairs its own model by observing or probing the learner, and the model is wrong because the teacher does not know how the learner updates. In our setting the model is wrong because other teachers have taught, and the learner initiates the repair in the teacher's own action space. We use machine teaching to ground the \emph{understanding statement}: the robot learner acts as a machine teacher toward the human, selecting the preference constraint that most efficiently drives the human's model of the learner toward the learner's true belief. The target here is not the reward but the learner's own belief state, and the learner, not the human, runs the selection.

\paragraph{Mental models and model repair}
Work on mental models in human-robot teams studies how each party models the other~\cite{tabrez-crr20,Andrews2023SharedMentalModel}, and most model-repair methods target the human's model of the agent. Model reconciliation explains a plan by communicating the model updates that make the plan optimal in the human's model~\cite{chakraborti-ijcai17}; robots select demonstrations so that a human can infer the robot's objective~\cite{huang-auro19}, and legible motion makes a robot's goal easy to infer from its behavior~\cite{dragan-hri13}. Cooperative IRL (CIRL) frames human and robot as jointly optimizing a shared reward that the human can teach rather than merely demonstrate~\cite{hadfield-menell-neurips16}. These methods repair what the human believes about the agent's objective or plan. Transparency work instead exposes the learner's state during teaching: a learner can reveal its uncertainty through its behavior~\cite{thomaz-icdl06}, so that teachers can direct examples at what it has not yet learned. Such signals report the learner's own state without modeling what the teacher already believes; our mean statement works the same way. The nearest work is Habibian et al.~\cite{habibian-thri22}, whose robot chooses preference queries that both learn the reward and reveal what it has learned. Their queries remain learner-guided, so a single query serves both purposes; our learner addresses a teacher who chooses the pairs, and it selects each statement against that teacher's model of the learner. Understanding statements thus repair the teacher's model of what the learner currently knows, which goes stale when other teachers teach, and they are expressed as preference pairs, the same form the teacher teaches with.

\paragraph{Theory-of-Mind in HRI} 
Recursive mental-state inference underlies action understanding \cite{baker-cognition09} and pedagogical reasoning, where teacher and learner model one another \cite{shafto-cogpsych14}. Interactive POMDPs formalize the nested belief modeling this requires \cite{gmytrasiewicz-doshi-jair05}, extended to treat communication itself as a belief-changing action selected by decision-theoretic planning \cite{gmytrasiewicz-jair20}. Most directly, Callaghan et al.~\cite{callaghan-arix25,callaghan-cogsci26} introduce \emph{understanding statements}, by which a ToM-2 robot learner corrects a human teacher's erroneous beliefs about what the learner knows, with a human study~\cite{callaghan-cogsci26} showing that statements chosen under a model of the teacher's possible cognitive biases elicit more informative teaching. We adopt their mechanism but differ in form and setting: our understanding statement is a preference constraint in the same action space the teacher uses to teach, rather than a linguistic statement over an enumerated feature set, and our teacher-model error arises from staleness across alternating teachers rather than a single teacher's cognitive bias.

\section{Problem Formulation}
\label{sec:formulation}

\subsection{Reward Learning from Preferences}
We consider a linear reward model $r(\xi) = x \cdot \psi(\xi)$ where $\psi(\xi) \in \Rd$ is a feature vector for trajectory $\xi$ and $x \in \Sd$ is a reward weight vector on the unit hypersphere; the true weight vector $\xstar$ is unknown to the learner.

A preference pair $(\xi_A, \xi_B)$ induces a halfspace constraint: the (noiseless) teacher prefers $\xi_A$ iff $(\psi(\xi_A) - \psi(\xi_B)) \cdot  x \geq 0$, i.e.\ $a \cdot x \geq 0$ where $a = \Delta\psi / \|\Delta\psi\|$.

The learner maintains a belief $\Bt = \{(x_i, w_i)\}_{i=1}^N$ with $x_i \in \Sd$, $\sum_i w_i = 1$, updated via spherical-cap particle filtering.
Belief alignment is:
\begin{equation}
  \phit = \mathbb{E}_{x \sim \Bt}[x \cdot \xstar] = \bar{x}_t \cdot \xstar,
  \quad \bar{x}_t = \textstyle\sum_i w_i x_i.
\end{equation}

The reason we use the expected belief alignment, $\phi$, instead of the mean belief alignment, $\rho=\frac{\bar{x}_t}{\|\bar x_t\|}\cdot \xstar$, is to incorporate uncertainty. A single halfspace constraint aligned with $\xstar$ would lead to a perfect mean belief alignment, $\rho=1.0$, yet the learner would be left with high uncertainty. 

\subsection{Teacher-Model Drift}
We model the teacher as an agent with (i) knowledge of the true reward $\xstar$, (ii) an explicit \emph{model of the learner} $\BT$ (teacher $T$'s model of the learner's current belief), and (iii) a curriculum policy that maps $\BT \to (\xi_A, \xi_B)$. The teacher's goal is to maximize $\phit$ within a fixed teaching budget. We call any divergence of $\BT$ from $\Bt$ \emph{teacher-model drift}. Drift arises whenever the learner's belief changes in a way the teacher did not observe or predict: another teacher teaches in its absence, or the learner updates differently than the teacher assumes. Under drift, the teacher spends rounds on constraints the learner has already absorbed, which add little information. We study the first source, alternating teachers, because the drift it creates can be accounted for exactly (Lemma~\ref{lem:ledger}).

We distinguish three beliefs. $\Bt$ is the learner's belief. $\BT$ is teacher $T$'s model of the learner, updated only with the pairs $T$ issued and the statements $T$ received; the teacher selects pairs against it. The learner cannot observe $\BT$, so it maintains $\BhT$, its model of the teacher's model, by applying the same constraints. This second-order belief is what ToM-2 statements are scored against. In simulation $\BhT = \BT$ by construction; with human teachers, this assumes the teacher updates on its own pairs as the learner does and absorbs each statement.

\subsection{Understanding Statements}
An understanding statement is a constraint of the form ``the learner's belief is consistent with $s \cdot x \geq 0$'', for some algorithmically selected $s \in \Sd$, that the learner communicates to the teacher. The statement is an ordered pair of trajectories chosen to communicate the learner's belief. Choosing it is a machine-teaching problem that mirrors the teacher's own: the learner selects an ordered demonstration pair for the teacher, as the teacher selects pairs for the learner.

\section{Methods}
\label{sec:methods}
\subsection{Demonstration Pair Construction}
\label{sec:pair_construction}
Both pair construction strategies decompose similarly: they generate a candidate set of halfspace directions, then select among them by the belief mass each removes. For a constraint $a \in \Sd$, let $V(a) = \sum_{i:\, a\cdot x_i < 0} w_i$ be the removed mass. A valid constraint down-weights the violating particles and renormalizes the belief; the analysis of Section~\ref{sec:analysis} treats this as a hard cut that zeroes them. The removed mass, counted in bits, is the part of a cut's effect that accumulates across rounds (Section~\ref{sec:round}), so $V(a)$ is the selection score.

\paragraph{Learner-guided querying}
The learner-guided condition, which we call EVR after expected volume removal~\cite{sadigh-rss17}, selects the balanced cut that maximizes worst-case removed mass, using only its own belief:
\begin{equation}
  a_{\mathrm{EVR}} = \arg\max_{a \in \Sd} \min\!\bigl(V(a),\, V(-a)\bigr).
  \label{eq:evr}
\end{equation}
This is representative of any belief-only acquisition rule: it cannot reference $\xstar$, and it shares the one-bit ceiling of Section~\ref{sec:ceiling} with information gain.

\paragraph{Teacher-guided learning}
The teacher knows $\xstar$, which changes the selection objective. EVR must hedge against both oracle answers, hence the max-min in $\min(V(a),V(-a))$; the teacher knows the answer places $\xstar$ in the retained halfspace, so it maximizes removed mass directly,
\begin{equation}
  a_T = \arg\max_{a:\; a\cdot\xstar \ge 0} V^T(a),
\end{equation} 
scored against its model of the learner $\BT$. This is the same removed-mass quantity, but freed from the balanced-cut constraint that target-agnostic selection is forced into. Candidates are drawn from counterfactual directions toward $x^*$ \cite{lee-iros22}, $a_k = (x^* - \mathrm{cf}_k)/\|x^* - \mathrm{cf}_k\|$ with $\mathrm{cf}_k \sim \BT$, augmented by the top-$M$ covariance eigenvectors as a floor when counterfactual sampling degrades at high dimension.

\subsection{Understanding Statement Generation}
\label{sec:ustate}

An understanding statement is a preference constraint $s \cdot x \geq 0$, $s \in \Sd$, emitted by the learner and applied by the teacher to its model of the learner $\BT$. It is the same action type the teacher uses to teach; the channel is symmetric. Selecting a statement is itself a machine-teaching problem whose objective is to reduce teacher-model drift. The two selection policies differ only in whose belief the learner reasons about when selecting $s$.

\paragraph{Mean statement}
The learner reasons about its own belief and reports the constraint that best summarizes it: the halfspace centered on the belief mean,
\begin{equation}
  s_\mathrm{mean} = \bar{x}_t / \|\bar{x}_t\|.
  \label{eq:mean}
\end{equation}
This statement contains no personalization for the teacher, simply communicating ``this is what I believe''.

\paragraph{ToM-2 statement}
\label{sec:tom2}
The learner reasons about the teacher's model \emph{of the learner} and, among the statements it holds, selects the one that removes the most mass from $\BhT$:
\begin{equation}
  s_\mathrm{ToM2} = \arg\max_{s:\; V(s) = 0}\; \hat{V}^T(s),
  \label{eq:tom2}
\end{equation}
where $\hat{V}^T(s)$ is the mass the statement removes from the learner's model of the teacher's model and $V(s)$ is the mass it removes from the learner's belief $\Bt$. The constraint $V(s) = 0$ restricts the learner to statements consistent with its own belief, and the model mass removed is the drift the statement repays (Section~\ref{sec:liability}). This statement is personalized: ``this is the correction that best fixes what you believe about me.''

\paragraph{When the two differ}
The policies separate when the teacher's model is wrong in a specific direction, as it is after a co-teacher has taught. The mean statement's halfspace removes only model mass lying more than $90^\circ$ from the learner's mean, which a concentrated model barely has, while the ToM-2 statement tilts toward the stale region and removes it (Fig.~\ref{fig:statement_compare}). Under hard cuts, the mass a consistent statement removes from the model is exactly the drift it repays (Lemma~\ref{lem:ledger}), so the ToM-2 objective is a drift-repayment objective.

\begin{figure}[t]
  \centering
  \includegraphics[width=\linewidth]{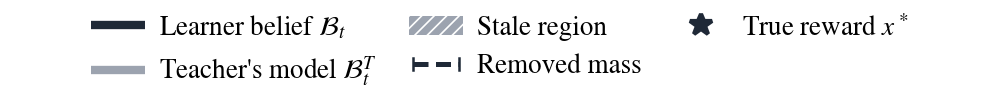}\par\vspace{2pt}
  \begin{minipage}[t]{0.49\linewidth}
    \centering
    \includegraphics[width=\linewidth]{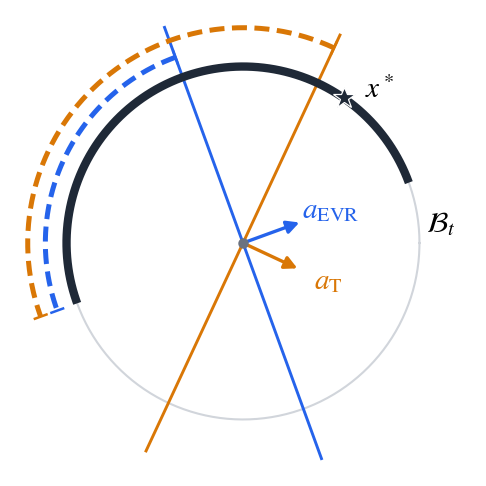}\par
    (a) Teaching cuts
  \end{minipage}\hfill
  \begin{minipage}[t]{0.49\linewidth}
    \centering
    \includegraphics[width=\linewidth]{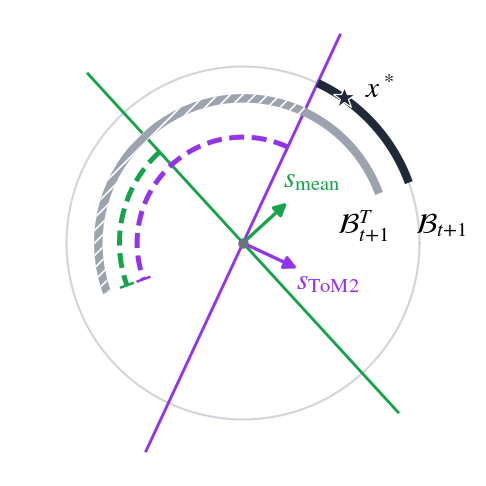}\par
    (b) Statements
  \end{minipage}
  % \caption{Cuts and understanding statements in $d{=}2$. Each belief is uniform on its arc, and dashed arcs mark the mass a cut removes. \textbf{(a)} The learner-guided query must split $\Bt$ evenly, while the teacher, knowing $\xstar$, can cut close to it and remove most of the belief. \textbf{(b)} A second teacher missed that cut, so its model $\mathcal{B}^T_{t+1}$ still holds the stale region. The mean statement removes only the edge of that region; the ToM-2 statement removes all of it, fully resynchronizing the model.}
  \caption{Cuts and understanding statements in $d{=}2$. Each belief is uniform on its arc; dashed arcs mark removed mass. \textbf{(a)} Knowing $\xstar$ lets the teacher remove more than the learner can: the learner-guided query removes half of $\Bt$ (1 bit), the teacher-guided cut three quarters (2 bits). \textbf{(b)} A second teacher missed $a_{\mathrm{T}}$, so its model lags by $D_{t+1} = 2$ bits. Targeting the stale region lets the ToM-2 statement repay all 2 bits, while the mean statement repays 0.68.}
  \Description{Two circle diagrams in two dimensions. In panel a, the learner's belief is an arc of the circle near the true reward. The learner-guided cut bisects the arc, removing half of it, while the teacher-guided cut passes just beside the true reward and removes three quarters of the arc. In panel b, a second teacher missed that cut, so its model still covers the whole old arc, most of which is a hatched stale region. The mean statement's cut removes only a small part of the stale region, while the ToM-2 statement's cut removes all of it.}
  \label{fig:statement_compare}
\end{figure}

\section{Analysis: Why Teaching Beats Querying and How Drift Erodes It}
\label{sec:analysis}
We can frame belief alignment in terms of information theory by quantifying the teacher-learner communication in bits. Section~\ref{sec:round} factors the change in alignment of a round into a volume removal term and a quality aim term. Sections~\ref{sec:ceiling} and~\ref{sec:bound} verify two known results in our setting: a learner choosing its own queries gains at most one bit per round in expectation while a target-aware teacher is not so limited~\cite{goldman-jcss95,zhu-survey18}, and reaching high alignment requires a number of bits that grows linearly in $d$, the dimensionality of the rule vector. Section~\ref{sec:liability} outlines the problem that understanding statements solve: the score a teacher uses to find a productive pair is discounted by the drift of its model of the learner.

\paragraph{Setup.} Both pair construction strategies (Section~\ref{sec:pair_construction}) score a halfspace direction by the belief mass it removes. Learner-guided querying scores against the learner's belief $\Bt$. Teacher-guided learning also sees the target $\xstar$, but it sees the learner only through its model $\BT$. We analyze the exact posterior under hard cuts: the prior is uniform on $\Sd$ and every constraint is noiseless, so $\Bt$ is uniform on the region $R_t \subseteq \Sd$ that satisfies all constraints so far. With $\sigma$ the normalized surface measure, the cut along $a$ removes mass $\Va = \sigma(R_t \cap \{a\cdot x<0\})/\sigma(R_t)$. Sections~\ref{sec:round}--\ref{sec:bound} assume a synchronized model, $\BT = \Bt$, and Section~\ref{sec:liability} drops that assumption.

% ─────────────────────────────────────────────────────────────────────────────
\subsection{What One Round of Teaching Buys}
\label{sec:round}
Each answered pair keeps the belief on the preferred side of a hyperplane and discards the rest. What that buys depends on how much is discarded and its proximity to the true rule.

Let $\mu^{+}(a)$ and $\mu^{-}(a)$ be the expected alignment, $\mathbb{E}_{x\sim\Bt}[\,x\cdot\xstar\,]$, of the mass the pair keeps ($a\cdot x \ge 0$) and the mass it removes ($a\cdot x < 0$). The current alignment is their mass-weighted average, $\phit = (1 -\Va)\,\mu^{+}(a) + \Va\,\mu^{-}(a)$, and after the update only the kept part remains, so $\phi_{t+1} = \mu^{+}(a)$. Subtracting the two,
\begin{equation}
  \Deltaphi \;=\; \frac{\Va}{1-\Va} \Qa\;=\; \lambda(a)\,\Qa,
  \qquad \Qa = \phit - \mu^-(a).
  \label{eq:dphi}
\end{equation}
The gain is a \emph{removal factor} $\lambda(a)$, set by how much mass the pair removes, times an \emph{aim} term $\Qa$, set by which mass it removes. The removal factor can be written in terms of the removed mass or in bits:
\begin{equation}
  \lambda(a) \;=\; \frac{\Va}{1-\Va} \;=\; 2^{\,b(a)} - 1,
  \qquad b(a) = -\log_2\!\big(1-\Va\big).
  \label{eq:lambda}
\end{equation}
A cut that halves the belief delivers one bit ($\lambda = 1$), and a cut that removes three quarters delivers two bits ($\lambda = 3$). We count in bits for two reasons. For an answered query, $b$ is the information the answer conveys, and its expectation is what information-gain acquisition maximizes~\cite{biyik-corl19}. Bits also add across rounds: each pair keeps a fraction of the remaining belief, so the fractions multiply, and their logarithms add:
\begin{equation}
  \log_2 \frac{1}{\sigma(R_K)} \;=\; \sum_{t<K} b_t .
  \label{eq:bits_add}
\end{equation}
Aim does not accumulate this way. Orienting each pair so that $\xstar$ is retained makes $\Qa \ge 0$ under the uniform prior, but once the belief is anisotropic, a cut can retain $\xstar$ and still remove better-aligned mass, so $\Qa$ can be negative. In short, bits measure how far a pair narrows the belief and add up over rounds; aim measures whether the narrowing moves the belief toward $\xstar$.

% ─────────────────────────────────────────────────────────────────────────────
\subsection{The Learner Ceiling and What Target Awareness Adds}
\label{sec:ceiling}
A learner that picks its own pair does not know which way the answer will go. If $\xstar \sim \Bt$ and a pair keeps mass $m$ on its preferred side, the answer delivers $-\log_2 m$ bits with probability $m$ and $-\log_2(1-m)$ bits otherwise, so the expected bits are the binary entropy $H(m) \le 1$, with equality only for the balanced cut $m = \tfrac12$. The bound uses only that the pair is chosen without $\xstar$, so it binds every belief-only acquisition rule; under a noiseless oracle, volume removal and information gain both select the balanced cut~\cite{sadigh-rss17,biyik-corl19}, as EVR~\eqref{eq:evr} does. Summed over rounds, the ceiling yields the classical lower bounds for learning from binary queries~\cite{dasgupta-nips05}.
 
A teacher who knows $\xstar$ is not bound by this ceiling. Both orientations of a hyperplane through $\xstar$ retain $\xstar$, and their removed masses sum to one, so the teacher can always remove at least half of the belief, and more whenever that hyperplane does not bisect it; tilting the cut slightly so that $\xstar$ is strictly retained preserves this, because $V$ is continuous. That a target-aware teacher needs fewer examples than a learner choosing its own queries is a founding observation of machine teaching~\cite{goldman-jcss95}. Zhu et al.~\cite{zhu-survey18} illustrate it with a one-dimensional threshold: a teacher fixes it with two examples, while an active learner needs a number of queries that grows logarithmically in the target precision. We restate the separation in bits only to verify that it holds in our setting. How far a teacher clears one bit depends on the shape of the belief and on which directions the domain can realize as pairs.

% ─────────────────────────────────────────────────────────────────────────────
\subsection{Bits Bound the Rounds to Alignment}
\label{sec:bound}
High alignment confines the belief to a small region of the sphere, and by~\eqref{eq:bits_add} a small region costs many bits. In high dimension, the directions within a fixed angle of $\xstar$ fill a fraction of the sphere that shrinks exponentially in $d$, so the number of halvings needed to reach them grows linearly in $d$.

\begin{proposition}[Bit budget]
\label{prop:bound}
Fix $\delta \in (0,\tfrac12)$. If $\phi_K \ge 1-\delta$, then
\[
  \sum_{t<K} b_t \;\ge\; B_d(\delta) \;=\; \frac{d\,(1-2\delta)^2}{2\ln 2} \;-\; 1 .
\]
\end{proposition}

\begin{proof}
  Let $\varepsilon = 1-2\delta$ and let $C = \{x \in \Sd : x\cdot\xstar \ge \varepsilon\}$ be the cap of directions near $\xstar$. The quantity $1 - x\cdot\xstar$ is nonnegative, its mean under $\mathcal{B}_K$ is $1-\phi_K \le \delta$, and it exceeds $2\delta$ outside $C$. For $x \sim \mathcal{B}_K$, Markov's inequality gives
  \[
    \Pr[x \notin C] \;\le\; \Pr[1 - x\cdot\xstar \ge 2\delta] \;\le\; \frac{\mathbb{E}[1 - x\cdot\xstar]}{2\delta} \;=\; \frac{1-\phi_K}{2\delta} \;\le\; \tfrac12 ,
  \]
  so $\mathcal{B}_K(C) = 1 - \Pr[x \notin C] \ge \tfrac12$. Because $\mathcal{B}_K$ is uniform on $R_K$, $\mathcal{B}_K(C) = \sigma(R_K\cap C)/\sigma(R_K)$, and the cap bound $\sigma(C) \le e^{-d\varepsilon^2/2}$ ~\cite[Lemma~2.2]{ball97} gives

  \[
    \sigma(R_K) \;=\; \frac{\sigma(R_K\cap C)}{\mathcal{B}_K(C)} \;\le\; 2\,\sigma(R_K\cap C) \;\le\; 2\,\sigma(C) \;\le\; 2\,e^{-d\varepsilon^2/2}.
  \]
  Taking $\log_2$ and applying~\eqref{eq:bits_add},
  \[
    \sum_{t<K} b_t \;=\; \log_2\frac{1}{\sigma(R_K)} \;\ge\; \log_2\frac{e^{d\varepsilon^2/2}}{2} \;=\; \frac{d\,\varepsilon^2}{2\ln 2} - 1 \;=\; B_d(\delta). \qedhere
  \]
\end{proof}
The budget grows linearly in $d$ and does not depend on how the cuts were chosen, so cuts that average $\bar{b}$ bits per round need $K \ge B_d(\delta)/\bar{b}$ rounds to reach $\phi_K \ge 1-\delta$. A learner-guided rule averages at most one bit per round in expectation (Section~\ref{sec:ceiling}), so it needs at least $B_d(\delta)$ rounds, a number linear in $d$. For a teacher the bound is only a floor: it can average $\bar{b} > 1$, but the budget is necessary and not sufficient, since a cut can remove much of the belief and leave the rest poorly aimed. The budget is also loose: $B_d(\delta) \le d/(2\ln 2) - 1$ for every $\delta$, so it captures the linear growth in $d$ but not the dependence on $\delta$, and we use it only for the former.

% ─────────────────────────────────────────────────────────────────────────────
\subsection{The Advantage Is Scored on a Model}
\label{sec:liability}
So far the teacher scored pairs on the learner's actual belief. In general it scores them on its model of the learner. When that model lags, part of the belief a pair appears to remove is already gone, removed by teaching the model never saw. The teacher selects $a_T = \arg\max_{a:\, a\cdot\xstar\ge 0} V^T(a)$, where $V^T$ is the mass a cut removes from its model $\BT$, but the constraint updates $\Bt$. Let $\RT$ be the region of the model and $b^T(a) = -\log_2(1-V^T(a))$. A model that lags the learner holds a subset of the learner's constraints, so $R_t \subseteq \RT$. We measure the lag in bits,
\[
  D_t \;=\; \log_2 \frac{\sigma(\RT)}{\sigma(R_t)} \;=\; \mathrm{KL}\big(\Bt \,\big\|\, \BT\big),
\]
which counts the teaching the learner has absorbed and the model has not seen. Learner-guided querying scores on $\Bt$ itself, so for it $D_t = 0$.

\begin{proposition}[Drift discounts the teacher's bits]
\label{prop:drift}
Let $R_t \subseteq \RT$. For any cut $a$ that retains $\xstar$, the bits it removes from the learner's belief, $b(a)$, and from the model, $b^T(a)$, satisfy
\[
  b(a) \;\ge\; b^T(a) - D_t ,
\]
with equality when the retained side contains no part of $\RT \setminus R_t$.
\end{proposition}

\begin{proof}
Let $H$ be the retained halfspace, so $2^{-b} = \sigma(R_t \cap H)/\sigma(R_t)$ and $2^{-b^T} = \sigma(\RT \cap H)/\sigma(\RT)$. Since $R_t \subseteq \RT$, the model retains at least the learner's retained region, and multiplying and dividing by $\sigma(R_t)$ gives
\[
  2^{-b^T}
  \;\ge\; \frac{\sigma(R_t \cap H)}{\sigma(\RT)}
  \;=\; \frac{\sigma(R_t \cap H)}{\sigma(R_t)}\cdot\frac{\sigma(R_t)}{\sigma(\RT)}
  \;=\; 2^{-b}\, 2^{-D_t},
\]
with equality exactly when $\sigma\big((\RT \setminus R_t)\cap H\big) = 0$. Taking $-\log_2$ of both sides gives $b^T \le b + D_t$.
\end{proof}

A cut whose removed model mass is entirely stale, $\RT \setminus R_t$, spends its scored bits on mass the learner has already discarded and delivers none of them. The proposition fixes the direction of this effect, not its size (Section~\ref{sec:scope}). How $D_t$ itself evolves is exact bookkeeping: a constraint that removes $b$ bits from a region multiplies that region's measure by $2^{-b}$, so every event in co-teaching moves $D_t$ by a known amount.

\begin{lemma}[Drift ledger]
\label{lem:ledger}
Under hard cuts, let $R_t \subseteq \RT$. Each event below preserves this containment and changes $D_t$ by the stated amount.
\begin{enumerate}
  \item[(i)] A cut $c$ applied to the learner but not to the model, such as a co-teacher's: $+\,b(c)$.
  \item[(ii)] The teacher's own cut $a$, applied to both: $+\,b(a) - b^T(a)$.
  \item[(iii)] An understanding statement $s$ applied to the model, provided it is consistent with the learner, $R_t \subseteq H_s$: $-\,b^T(s)$.
\end{enumerate}
In particular a consistent statement repays at most the drift, $b^T(s) \le D_t$, with equality exactly when $\RT \cap H_s = R_t$ up to measure zero.
\end{lemma}

\begin{proof}
Write $D_t = \log_2 \sigma(\RT) - \log_2 \sigma(R_t)$. Intersecting a region with a halfspace that removes $b$ bits from it lowers the region's $\log_2 \sigma$ by exactly $b$.
(i) lowers $\log_2\sigma(R_t)$ by $b(c)$ and leaves the model unchanged, and the new learner region $R_t \cap H_c$ still lies inside $\RT$.
(ii) lowers $\log_2\sigma(R_t)$ by $b(a)$ and $\log_2\sigma(\RT)$ by $b^T(a)$, and intersecting both regions with the same halfspace keeps one inside the other.
(iii) lowers $\log_2\sigma(\RT)$ by $b^T(s)$ and leaves the learner unchanged, and because $R_t \subseteq H_s$ the new model $\RT \cap H_s$ still contains $R_t$.
Containment keeps the drift nonnegative, so $D_t - b^T(s) \ge 0$ after (iii), with equality exactly when $\sigma(\RT \cap H_s) = \sigma(R_t)$, that is, when $\RT \cap H_s = R_t$ up to measure zero.
\end{proof}

Drift therefore accumulates like a debt: it grows by the full bits of every cut the model misses. Item (ii) shows that a stale teacher partly resynchronizes on its own: a cut that wastes scored bits on stale mass, $b(a) < b^T(a)$, lowers $D_t$ by exactly the amount wasted. A statement makes the same repair without spending a teaching round. The ledger is what the ToM-2 objective~\eqref{eq:tom2} optimizes: among consistent statements, removing the most mass from the model is the same as repaying the most drift. This holds by construction, so the open question is empirical: whether repaid drift turns into faster learning in later rounds, which the ledger does not size (Section~\ref{sec:scope}). The mean statement~\eqref{eq:mean} is chosen without reference to the model, so when it is consistent it repays at most as much.

Proposition~\ref{prop:drift} and Lemma~\ref{lem:ledger} are exact only under hard cuts. Our simulated learner applies soft particle updates, which down-weight violating particles without removing them, so the containment $R_t \subseteq \RT$ and the bookkeeping do not hold exactly. The simulations in Section~\ref{sec:simulation} therefore measure drift with EMD and test the ledger's qualitative predictions, not its bit counts.

% ─────────────────────────────────────────────────────────────────────────────
\subsection{Limits of the Analysis}
\label{sec:scope}
\paragraph{Direction, not size.} Proposition~\ref{prop:drift} subtracts all of $D_t$ from a cut's scored bits, as if the model's entire stale region lay on the side the cut removes. When several co-teachers have taught in different directions, $D_t$ can exceed $b^T(a)$ by a wide margin and the bound becomes vacuous, even when the cut delivers most of its scored bits. We therefore use the proposition for the direction of the effect (drift can only cost bits, and repaying it can only help), not its size.

\subsection{Predictions}
\label{sec:predictions}
The analysis yields three predictions, tested in Section~\ref{sec:simulation}.
\textbf{(P1)} Drift lowers the teacher's advantage: Uninformed teachers, whose models drift, need more rounds than Ideal teachers, motivated by Proposition~\ref{prop:drift}.
\textbf{(P2)} Because the bit budget grows linearly in $d$ (Proposition~\ref{prop:bound}), the rounds learner-guided querying needs to reach a fixed alignment grow at least linearly in $d$, and so do a teacher's if its bits per round stay bounded.
\textbf{(P3)} ToM-2 statements reduce drift more than Mean statements (Lemma~\ref{lem:ledger}), and if the repaid bits are spent on later teaching, ToM-2-informed teachers outperform Mean-informed and Uninformed teachers.

\begin{figure*}[t]
  \centering
  % Panel A: Learning Curves
  \begin{minipage}{0.32\textwidth}
    \centering
    \includegraphics[width=\linewidth]{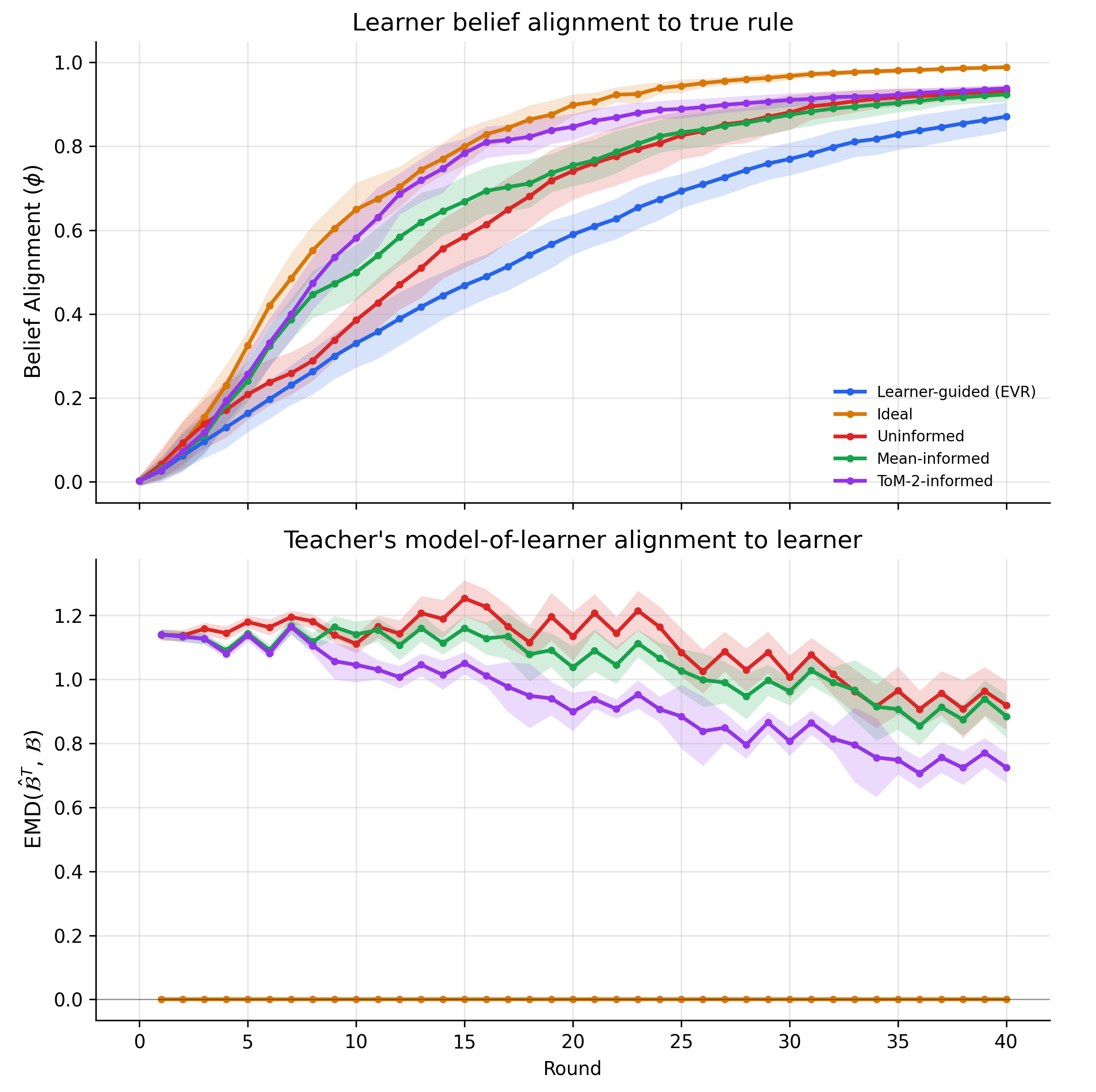}
    \centerline{(a) Alternating teachers ($d{=}20$)}
  \end{minipage}\hfill
  % Panel B: Dimensionality 
  \begin{minipage}{0.32\textwidth}
    \centering
    \includegraphics[width=\linewidth]{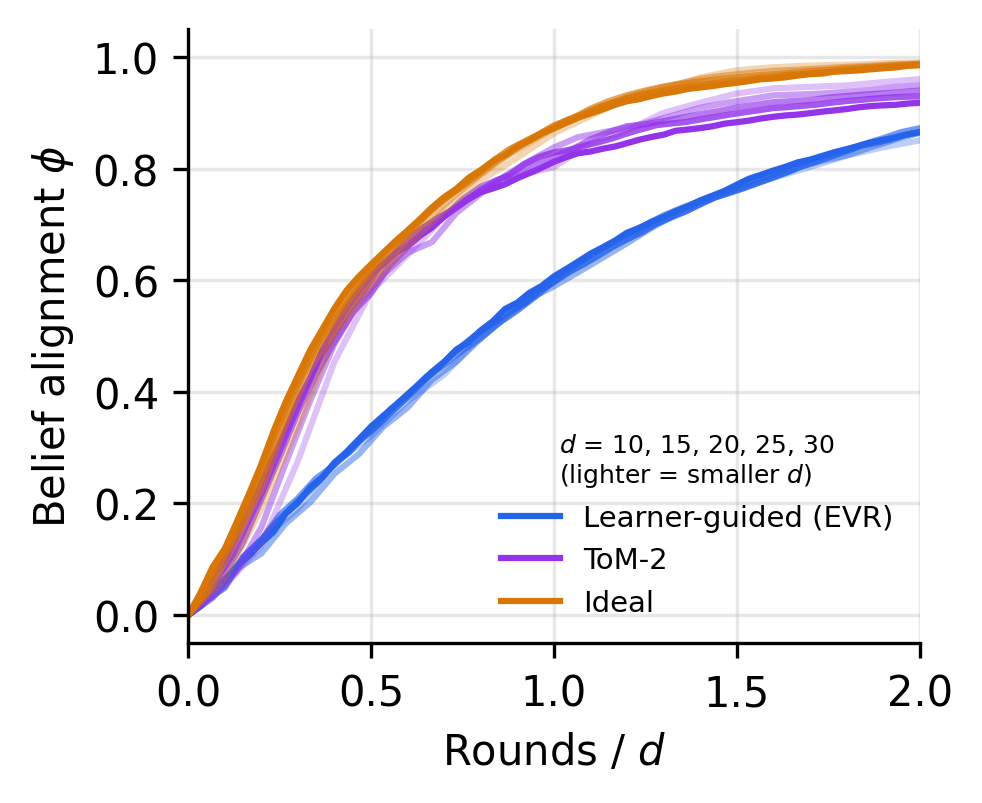}
    \centerline{(b) Teacher-guided dimensional advantage}
  \end{minipage}\hfill
  % Panel C: Dimensionality convergence
  \begin{minipage}{0.32\textwidth}
    \centering
    \includegraphics[width=\linewidth]{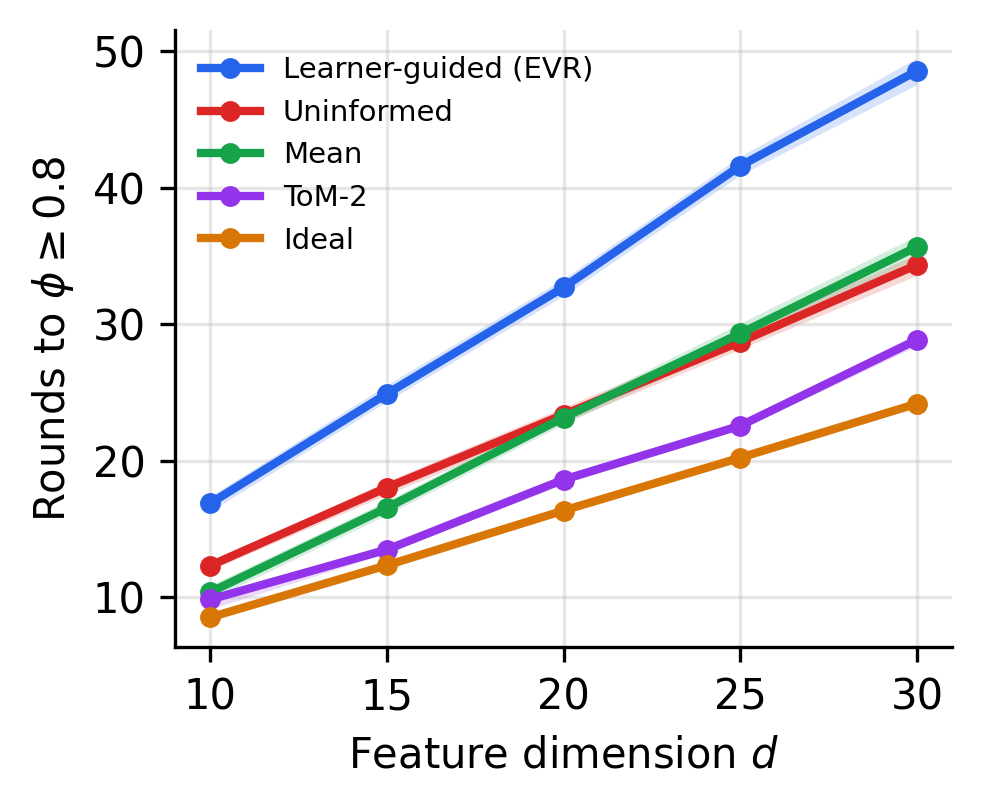}
    \centerline{(c) Rounds to convergence}
  \end{minipage}
  \vspace{0.2cm} % Adds a tiny buffer before the caption
  \caption{Simulation results over 50 trials. \textbf{(a)} Learner belief alignment $\phi_t$ (top) and teacher-model error EMD (bottom) under four alternating teachers; bands denote $\pm$1 std. \textbf{(b)} Belief alignment against rounds per dimension for $d \in \{10,\dots,30\}$ (lighter lines are smaller $d$); each condition's curves collapse, so rounds scale linearly in $d$. \textbf{(c)} Rounds for the alignment curve to reach $\phi \ge 0.8$ against $d$.}
  \Description{Three line plots. Panel a shows belief alignment over 40 rounds at dimension 20 for five conditions: Ideal rises fastest, ToM-2 informed is next, Mean informed and Uninformed are close, and learner-guided EVR is slowest. Below it, teacher-model error is zero for Ideal, oscillates highest for Uninformed, and falls most steadily for ToM-2 informed. Panel b plots alignment against rounds divided by dimension for dimensions 10 to 30; each condition's curves overlap. Panel c plots rounds to reach alignment 0.8 against dimension; all conditions grow linearly, steepest for EVR and shallowest for Ideal, with ToM-2 between Ideal and Uninformed.}
  \label{fig:combined_results}
\end{figure*}

% ── 6. Simulation Results ─────────────────────────────────────────────────────
\section{Simulation Results}
\label{sec:simulation}
All simulations use a halfspace particle filter on $\Sd$ updated via spherical-cap reweighting; the dimensionality study (Fig.~\ref{fig:combined_results}b--c) sweeps $d \in \{10, 15, 20, 25, 30\}$ with $N = 250d$ particles and $2d + 20$ rounds. We compare five conditions. \emph{Learner-guided (EVR)}: the learner selects balanced cuts from its own belief; no teacher model exists. \emph{Ideal}: all teachers' models of the learner are kept exactly synchronized every round, the no-drift ceiling. The remaining conditions share a multi-teacher protocol: four teachers take turns of two rounds each, and each teacher's model is updated only by the preferences that teacher itself issued, so a model goes stale between turns. \emph{Uninformed} teachers receive no statements; \emph{Mean-informed} and \emph{ToM-2-informed} teachers receive one understanding statement at the start of each turn, selected by the mean and ToM-2 policies of Section~\ref{sec:ustate}, respectively. Teacher-model drift is measured as the Earth Mover's Distance between $\hat{\mathcal{B}}_t^T$ and $\Bt$ with geodesic ground cost, which compares the full distributions rather than their means: two beliefs can share a mean and still differ in spread and in the mass they place on stale directions. The results average 50 independent trials, with $\xstar$ drawn uniformly on $\Sd$ per trial. Constraints are sampled directly on $\Sd$ without domain validity restrictions.

\subsection{Learning Curve Comparison}
Fig.~\ref{fig:combined_results}a (top) shows the alignment ordering Ideal $>$ ToM-2 $>$ Mean $>$ EVR. All teacher-guided conditions dominate the learner-guided baseline. Drift is costly: the Uninformed condition falls below the Ideal condition, as \textbf{(P1)} predicts. The bottom panel shows why. Under the Ideal condition model error is identically zero; without statements it spikes when teachers alternate; mean statements bound it; and ToM-2 statements lead to a steady decrease. The alignment ordering mirrors the model-error ordering, as \textbf{(P3)} predicts: ToM-2 statements reduce drift the most, and among the drifting conditions their teachers learn fastest.

\subsection{Scaling with Feature Dimensionality}
Fig.~\ref{fig:combined_results}b--c sweep the feature dimension. Plotting alignment against rounds per dimension (b) collapses each condition's curves onto one, so the rounds needed to reach a given alignment grow linearly in $d$ for learner- and teacher-guided learning alike, consistent with \textbf{(P2)}. The slopes differ (c): reaching $\phi \ge 0.8$ costs EVR about $1.6$ rounds per dimension and the Ideal teacher about $0.8$, so the rounds saved by teacher guidance also grow linearly in $d$. Proposition~\ref{prop:bound} bounds only the learner-guided rate, so this slope is an empirical property of our greedy teacher, not a prediction. Drift costs rounds at a rate of its own: Uninformed teachers need about $1.1$ rounds per dimension \textbf{(P1)}. ToM-2 statements recover much of this loss at every $d$ (about $0.95$ rounds per dimension), while Mean statements help at $d = 10$ but match or trail the Uninformed condition from $d = 20$ on, consistent with \textbf{(P3)}.

\begin{figure}[t]
  \centering
  \begin{minipage}{\linewidth}
    \centering
    \includegraphics[width=\linewidth]{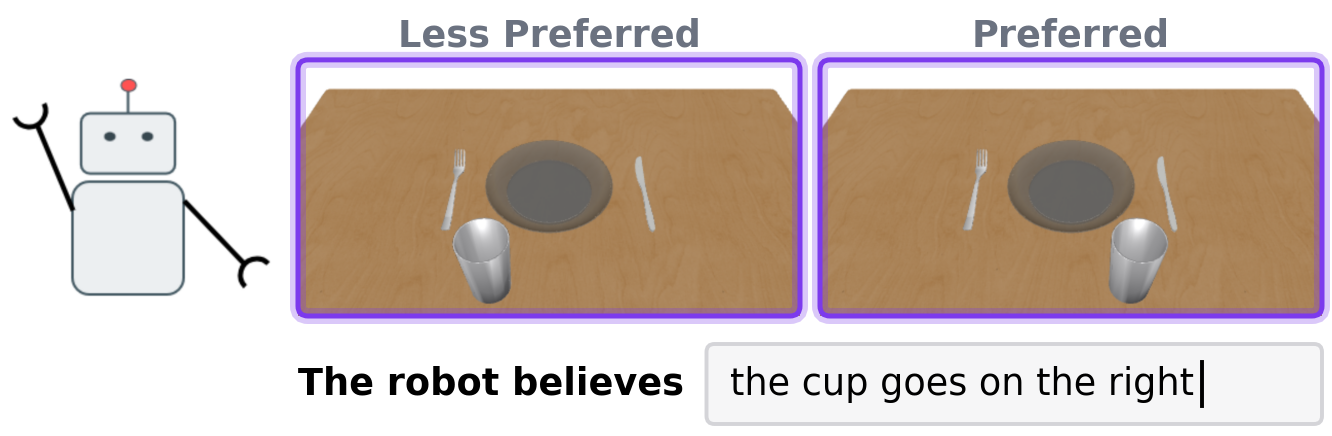}
    \vspace{2pt}
    \centerline{\textbf{(a) Robot's understanding statement}}
    \vspace{8pt}
  \end{minipage}
  \begin{minipage}{\linewidth}
    \centering
    \includegraphics[width=\linewidth]{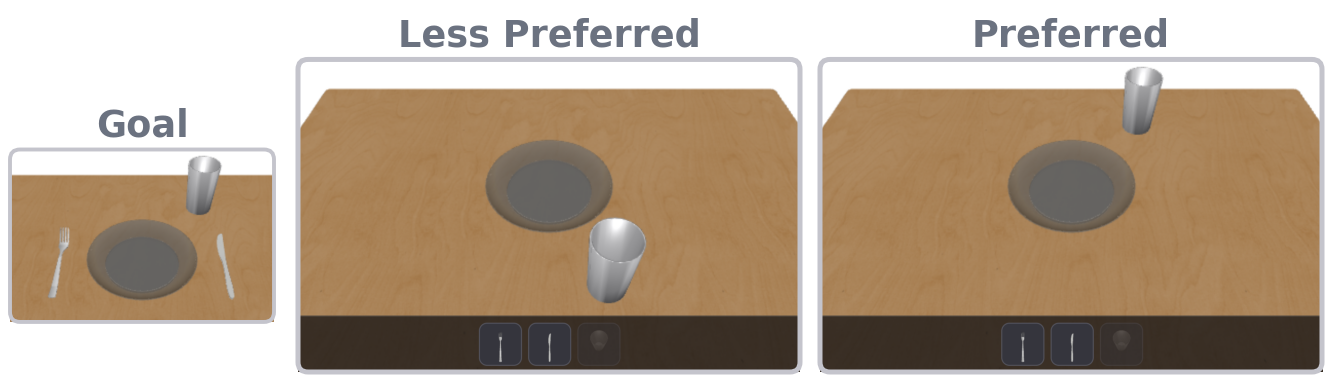}
    \vspace{2pt}
    \centerline{\textbf{(b) Participant's teaching pair}}
  \end{minipage}
  \caption{The study interface. \textbf{(a)} In the Mean and ToM-2 conditions, each participant round opens with an understanding statement: a preference pair. Here the robot shows that it has learned the cup belongs on the right of the plate. \textbf{(b)} The participant builds a teaching pair by dragging objects into two scenes. Here they teach what the statement did not show: the cup belongs behind the plate.}
  \Description{Screenshots of the study interface. Top: the robot's understanding statement, two tabletop scenes labeled less preferred and preferred, in which the cup moves from the front left to the front right of the plate, with the participant's restatement below: the cup goes on the right. Bottom: the participant's teaching pair beside the goal image, with the cup in front of the plate labeled less preferred and the cup behind the plate labeled preferred.}
  \label{fig:study_interface}
\end{figure}

\section{Human Study}
\label{sec:study}
\subsection{TidySim Domain}
Participants taught a robot learner where to place a fork, a knife, and a cup around a plate in a tabletop rearrangement simulator based on TidySim~\cite{pmlr-v164-kapelyukh22a}, deployed as a web application. The reward is linear in each object's position relative to the plate, lifted to $[\Delta x, \Delta z, -(\Delta x^2 + \Delta z^2)]$ so that a linear reward can encode a preferred position for each object ($d = 9$). Understanding statements are rendered as arrangement pairs by perturbing the learner's mean belief along the statement direction $s$. Each task showed the correct arrangement as a goal image, so the participant knew $\xstar$. The learner is the particle filter of Section~\ref{sec:formulation} with $N = 2{,}000$ particles.

\subsection{Design and Procedure}
The study used a within-subjects design with five conditions, each a block of 16 rounds. Three were teacher conditions analogous to the drifting conditions of Section~\ref{sec:simulation}: \emph{Uninformed}, \emph{Mean-informed}, and \emph{ToM-2-informed}. The other two were learner-guided EVR conditions, which are not analyzed here. In every block the participant and a simulated co-teacher alternated two rounds each, starting with the co-teacher, giving the participant 8 rounds per block. The co-teacher's demonstrations were hidden behind a screen that participants acknowledged before each of their turns, telling them that the robot had been taught in the meantime. Participants therefore never saw the co-teacher's pairs; in the statement conditions, its effect reached them only through the robot's understanding statements. Block order followed a Williams design, which balances both the position of each condition and which condition precedes it.

In the teacher blocks, the participant built two arrangements, a preferred one and a less preferred one, and the robot updated on the pair. In the Mean- and ToM-2-informed blocks, every teaching round began with an understanding statement chosen by the corresponding policy of Section~\ref{sec:ustate}, shown as two arrangements with the preferred one marked. To encourage comprehension, participants had to restate each statement in their own words by completing the sentence ``The robot believes \ldots'' before they could continue; ``not sure'' was accepted. The two statement conditions were identical in instructions and interface, so participants could not tell them apart.

Before the first block, a single tutorial with two practice objects covered every interaction used in the blocks. After each of their turns, participants rated how much the robot had learned about each object's location. After each block they rated mental demand, confidence in their teaching, and whether the co-teacher had changed the robot's behavior in ways they did not understand. After the two statement blocks they also rated how clearly and how usefully the statements revealed what the robot had learned.
 
\subsection{Participants}
The study was approved by our institution's IRB, and all participants gave informed consent. We recruited participants on Prolific and compensated them \$12/hour. Of 106 who completed the study, 5 were excluded for taking over 100 minutes, leaving $n = 101$ (47 female, 53 male, 1 non-binary; aged 18--24: 24, 25--34: 49, 35--44: 12, 45--54: 8, 55--64: 5, 65--74: 2, undisclosed: 1). 
% Participants reported extensive prior experience with robots or AI assistants (mean $5.9$ on a 1--7 scale, $SD = 0.95$), and the median session took 51 minutes.

\subsection{Metrics}
We report final belief alignment $\phi$ and robot regret, the reward lost by acting on a learned reward rather than the true one~\cite{wilde-iros20}, taken in expectation over the robot's belief as in~\cite{habibian-thri22}. With $\xi^*_x = \arg\max_\xi x\cdot\psi(\xi)$,
$$
  \mathrm{Regret}_t = \mathbb{E}_{x\sim\Bt}\big[\xstar\cdot\psi(\xi^*_{\xstar}) - \xstar\cdot\psi(\xi^*_x)\big].
$$
To isolate the participant's contribution, \emph{regret drop} sums the decrease in $\mathrm{Regret}_t$ over the participant's own rounds. Conditions are compared within participants with paired tests and effect size $d_z$.

\subsection{Results}
ToM-2 statements raised final alignment above both other conditions (Fig.~\ref{fig:human_results}a): relative to Uninformed teaching by $d_z = 0.30$ ($p_{\mathrm{Holm}} = .006$) and relative to Mean-informed teaching by $d_z = 0.24$ ($p_{\mathrm{Holm}} = .018$). The same ordering held for regret drop (Fig.~\ref{fig:human_results}b): ToM-2 statements removed significantly more regret than Uninformed teaching ($d_z = 0.24$, $p_{\mathrm{Holm}} = .038$) and slightly, but not significantly, more than Mean statements ($d_z = 0.07$, $p_{\mathrm{Holm}} = .476$). For both metrics a one-way repeated-measures ANOVA was significant (final alignment: $F(1.87, 187.5) = 4.59$, $p = .013$, Greenhouse--Geisser corrected for violated sphericity; regret drop: $F(2, 200) = 3.62$, $p = .029$), and the two planned comparisons, ToM-2 against each baseline, used paired $t$-tests with Holm correction.

\begin{figure}
  \centering
  \begin{minipage}{0.49\linewidth}
    \centering
    \includegraphics[width=\linewidth]{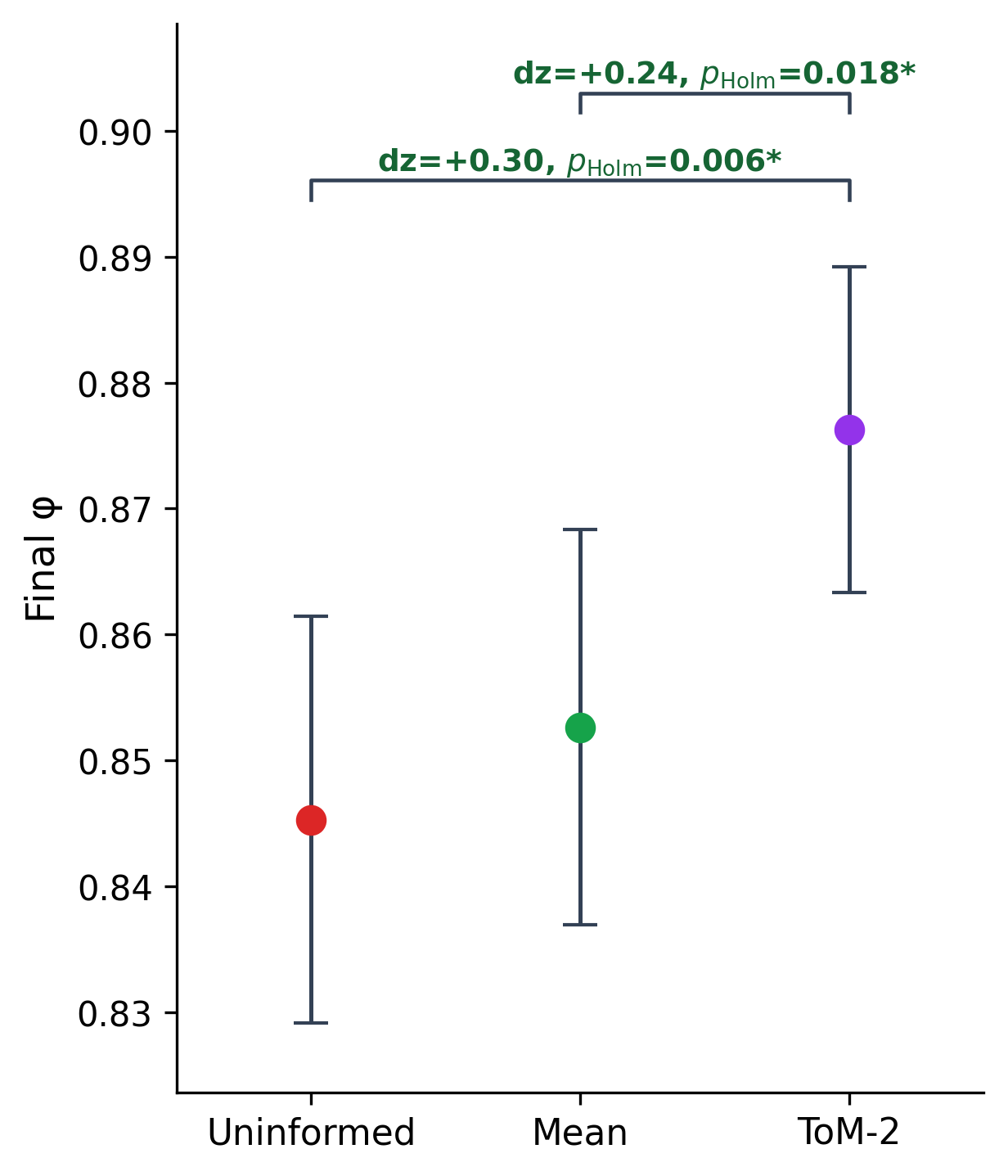}
    \centerline{(a) Final alignment}
  \end{minipage}\hfill
  \begin{minipage}{0.49\linewidth}
    \centering
    \includegraphics[width=\linewidth]{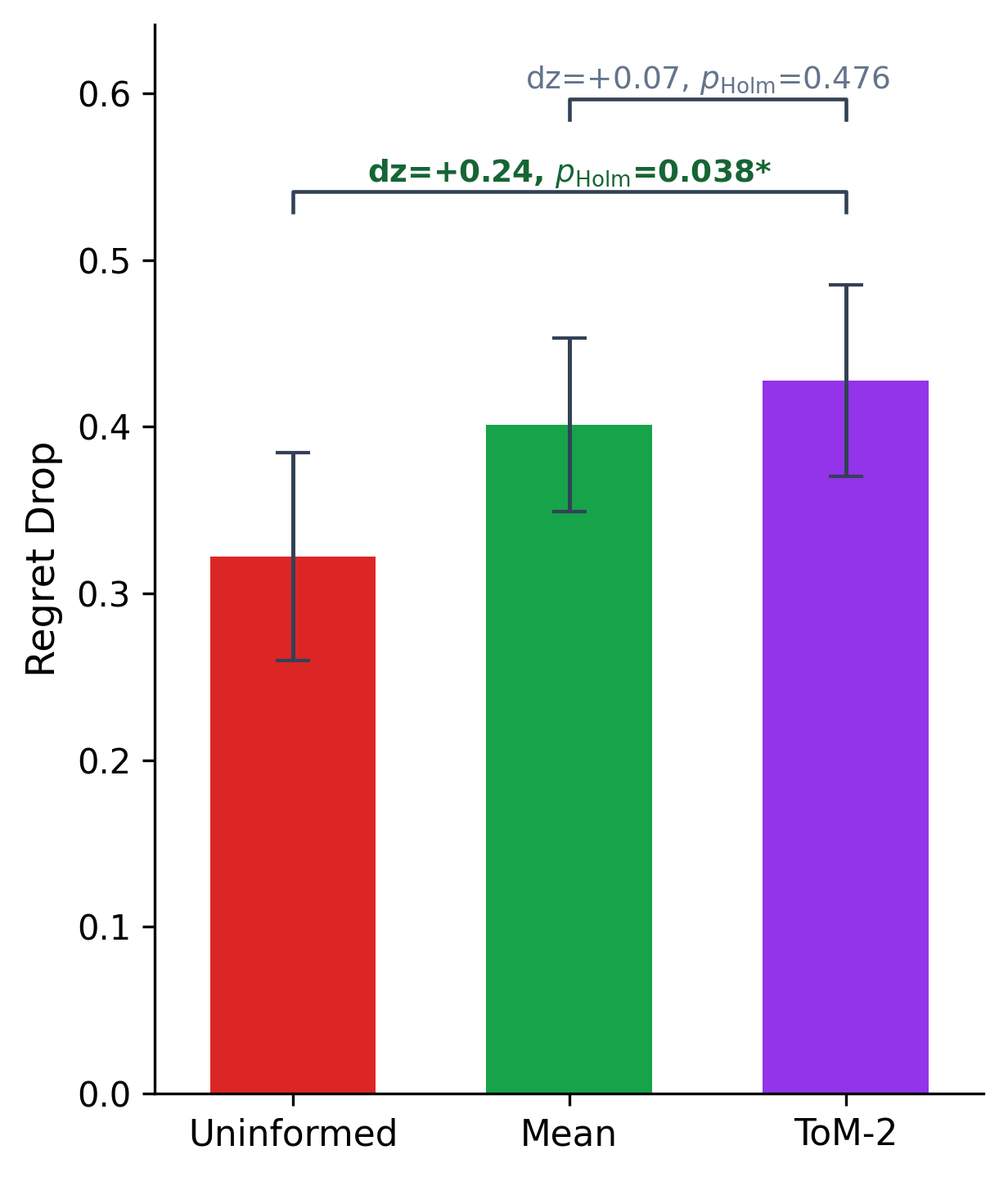}
    \centerline{(b) Regret drop}
  \end{minipage}
  \caption{Human study results by condition (within-subject means, 95\% intervals, $n = 101$). \textbf{(a)} Final belief alignment (axis clipped to focus on the difference). \textbf{(b)} Regret removed by the participant's own teaching.}
  \Description{A dot and bar chart with 95 percent within-subject intervals for 101 participants. Final belief alignment is about 0.85 for Uninformed and Mean-informed teaching and about 0.88 for ToM-2-informed teaching. Regret drop is about 0.32 for Uninformed, 0.40 for Mean-informed, and 0.43 for ToM-2-informed teaching.}
  \label{fig:human_results}
\end{figure}

\section{Limitations and Future Work}
\label{sec:future}
The human-study effects are small ($d_z \le 0.30$), and drift came from one source, a simulated Ideal co-teacher with a synchronized model. Each limitation suggests an extension. The reward here is linear in hand-designed features of a tabletop task with $d = 9$; richer domains need reward features that are not hand-designed. We also assumed that teacher and learner share a feature representation. Finally, drift can arise within a single teacher whose updates differ from the learner's, as under cognitive bias~\cite{callaghan-arix25,callaghan-cogsci26}. The drift ledger assumes the model holds a subset of the learner's constraints, which a biased update can violate, so both the analysis and the statement policy need extending to cover this case.

\section{Conclusion}
A teacher who knows the target can outpace learner-guided querying, but it scores its pairs on a model of the learner, and with several teachers that model drifts. Measured in bits, drift follows an exact ledger under hard cuts: pairs the teacher misses add to it, and understanding statements consistent with the learner subtract from it. In simulation, ToM-2 statements, which repay the most drift, recovered more of the teacher's advantage than statements of the learner's mean belief. With human teachers, ToM-2 statements raised final alignment over both uninformed and mean-informed teaching, and increased the regret participants removed through their own teaching  relative to uninformed teaching.

%%% The following command should be issued somewhere in the first column 
%%% of the final page of your paper.
\balance

%%%%%%%%%%%%%%%%%%%%%%%%%%%%%%%%%%%%%%%%%%%%%%%%%%%%%%%%%%%%%%%%%%%%%%%%

%%% The acknowledgments section is defined using the "acks" environment
%%% (rather than an unnumbered section). The use of this environment 
%%% ensures the proper identification of the section in the article 
%%% metadata as well as the consistent spelling of the heading.

\begin{acks}
This work was partially supported by a grant from SoftBank Group Corporation and Arm.
\end{acks}

%%%%%%%%%%%%%%%%%%%%%%%%%%%%%%%%%%%%%%%%%%%%%%%%%%%%%%%%%%%%%%%%%%%%%%%%

%%% The next two lines define, first, the bibliography style to be 
%%% applied, and, second, the bibliography file to be used.

\bibliographystyle{ACM-Reference-Format} 
\bibliography{sample}

%%%%%%%%%%%%%%%%%%%%%%%%%%%%%%%%%%%%%%%%%%%%%%%%%%%%%%%%%%%%%%%%%%%%%%%%

\end{document}